\documentclass{article}

    \PassOptionsToPackage{numbers, compress}{natbib}

\usepackage[final]{neurips_2023}
\usepackage[utf8]{inputenc} 
\usepackage[T1]{fontenc}    
\usepackage[hidelinks]{hyperref}       
\hypersetup{
    colorlinks = true,
    linkcolor = red,
    citecolor = blue
    }

\usepackage{url}            
\usepackage{booktabs}       
\usepackage{amsfonts}       
\usepackage{nicefrac}       
\usepackage{microtype}      
\usepackage{xcolor}         
\usepackage{wrapfig}
\usepackage{amsthm}
\newtheorem{theorem}{Theorem}
\newtheorem{definition}{Definition}
\newtheorem{lemma}{Lemma}
\usepackage{tikz}
\usetikzlibrary{positioning}
\usetikzlibrary{bayesnet}
\usetikzlibrary{fit}
\usetikzlibrary{shapes,arrows}
\usetikzlibrary{shapes.multipart}
\usetikzlibrary{arrows.meta}
\usetikzlibrary{fit}

\usepackage{amsmath}
\DeclareMathOperator*{\argmin}{arg\,min}
\def\*#1{\mathbf{#1}}
\def\Pa#1{\text{Pa}#1}
\def\f{f_{\textnormal{causal}}}
\def\g{g_{\textnormal{causal}}}
\def\h{h_{\textnormal{causal}}}

\DeclareMathOperator{\Ima}{Im}

\usepackage[capitalise]{cleveref}

\title{Invariance \& Causal Representation Learning: Prospects and Limitations}

\author{%
  Simon Bing$^{1}$ \quad Jonas Wahl\thanks{Equal contribution.} $^{\,1,2}$ \quad Urmi Ninad$^{*1,2}$ \quad Jakob Runge$^{2,1}$ \\
    \\
  $^{1}$Technische Universität Berlin, Berlin, Germany\\
  $^{2}$Institute of Data Science, German Aerospace Center, Jena, Germany\\
  \texttt{\{bing,wahl,urmi.ninad\}@tu-berlin.de}, \texttt{jakob.runge@dlr.de} \\
}

\begin{document}

\maketitle

\begin{abstract}
    In causal models, a given mechanism is assumed to be invariant to changes of other mechanisms. While this principle has been utilized for inference in settings where the causal variables are observed, theoretical insights when the variables of interest are latent are largely missing. We assay the connection between invariance and causal representation learning by establishing impossibility results which show that invariance alone is insufficient to identify latent causal variables. 
    Together with practical considerations, we use
    these theoretical findings to highlight the need for additional constraints in order to identify 
    representations by exploiting invariance.
\end{abstract}
\section{Introduction}
Inferring high-level causal variables from low-level measurements is a problem garnering increased attention in fields interested in understanding epiphenomena that cannot be directly measured and where controlled experiments are not possible due to practical, economic or ethical considerations, for instance in healthcare \citep{johansson_generalization_2022}, biology \citep{lopez_learning_2023} or climate science \citep{tibau_spatiotemporal_2022}. This problem of causal representation learning \citep{scholkopf_toward_2021} has been shown to be fundamentally underconstrained \citep{locatello_challenging_2019}, leading to various approaches exploring which assumptions lead to algorithms that identify the latent causal variables.

Recent works either restrict the underlying causal model \citep{lachapelle_disentanglement_2022,buchholz_learning_2023}, the transformation causal variables undergo 
\citep{ahuja_interventional_2023,lachapelle_synergies_2023}, or both \citep{squires_linear_2023}. They include interventional or counterfactual data \citep{ahuja_interventional_2023,zhang_identifiability_2023,squires_linear_2023,buchholz_learning_2023,brehmer_weakly_2022}, use supervisory signals such as time structure \citep{hyvarinen_nonlinear_2017,halva_hidden_2020,yao_learning_2021} or knowledge of intervention targets \citep{lippe_citris_2022,lippe_causal_2022}. 

We explore the applicability of another type of inductive bias for identifiable representation learning, namely the invariance of causal mechanisms \citep{peters_elements_2017}. 
\citet{haavelmo_probability_1944} first showed that causal variables lead to predictive models that are invariant under interventions,
and since causal representation learning is tasked with recovering precisely these variables, we investigate if and to which degree the principle of invariance can be used as a signal to recover latent causal variables from observations.

While invariance has been used for causal inference \citep{peters_causal_2016, buhlmann_invariance_2018, meinshausen_causality_2018}, none of these works considers the setting where we only have access to observations that are related to the underlying causal variables by some unknown transformation.
To the best of our knowledge, we present the first theoretical results pertaining to the identifiability of causal representations using the principle of invariance.

Our contributions are summarized as follows:
\begin{itemize}
    \item Drawing on the link between distributional robustness and causality, we formalize the setting in which we study the connection between invariance and causal representation learning.
    \item We establish impossibility results proving the necessity of additional assumptions to achieve identifiability.
    \item Based on these impossibility results and practical considerations, we contemplate further constraints and map out future work towards identifiable algorithms based on invariance.
\end{itemize}



\section{Problem setting}
\label{sec:prob}

\begin{wrapfigure}{r}{0.425\textwidth}
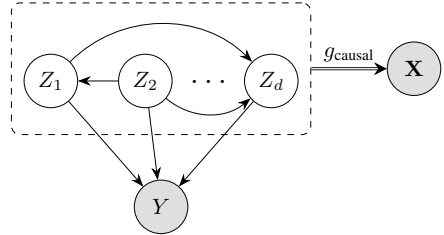

    \vspace{-13pt}
    \centering
    \resizebox{.42\textwidth}{!}{
    \tikz[latent/.append style={minimum size=0.85cm},obs/.append style={minimum size=0.85cm},det/.append style={minimum size=1.15cm}, wrap/.append style={inner sep=2pt}, on grid]{
        \node[latent](z_1){\(Z_1\)};
        \node[latent, right=1.5cm of z_1](z_2){\(Z_2\)};
        \node[latent, right=2cm of z_2](z_d){\(Z_d\)};
        \node[obs, below=2cm of z_1, xshift=1.75cm](y){\(Y\)};
        \node[right=of z_2, node font=\Large]{\(\dots\)};
        \edge[-{Stealth[length=2mm, width=1.5mm]}]{z_2}{z_1};
        \draw[-{Stealth[length=2mm, width=1.5mm]}] (z_1) to [out=45, in=135] node [above=-3pt](e1){} (z_d);
        \draw[-{Stealth[length=2mm, width=1.5mm]}] (z_2) to [out=-40, in=220] (z_d);
        \edge[-{Stealth[length=2mm, width=1.5mm]}]{z_1, z_2, z_d}{y};
        \plate[dashed]{SCM}{(z_1)(z_2)(z_d)(e1)}{};
        \node[obs, right=1.2cm of SCM.east](X){\(\*X\)};
        \draw[-{Stealth[length=2mm, width=2mm]}, double] (SCM.east) to node [above](e2){\(\g\)} (X)
    }
    }
    \caption{We consider an SCM with variables \((Z_1,\dots,Z_d)\) and \(Y\). Observed variables are represented by shaded nodes, indicating that we do not observe \((Z_1,\dots,Z_d)\), but only their transformation \(\*X=\g(\*Z)\).}
    \vspace{-25pt}
    \label{fig:dgp}
\end{wrapfigure}
Intuitively, our problem setting can be motivated by considering the prediction problem with a target \(Y\) and observations \(\*X\) in multiple environments. We assume that there is an underlying, causal representation of the observations---denoted with \(\*Z\)--- whose constituents are causes of the prediction target $Y$ and interact with each other through a structural causal model (SCM). This representation $\*Z$ is what we are interested in finding.
\paragraph{Notation.}
We denote scalar variables in normal face (\(x\)) and use bold face for vector-valued variables (\(\*x\)). We capitalize random variables (\(Y\)), and write the values they take in lower case (\(y\)). Matrices are denoted capitalized and bold (\(\*M\)) and are explicitly introduced to avoid confusion with vector-valued random variables. The sequence of integers from \(1\) to \(n\) is denoted with \([n]\).

\paragraph{Data generating process.}

Let \((Z_1, \dots, Z_{d+1})\) denote a set of random variables. W.l.o.g. we call \(Z_{d+1}\) the \emph{target} variable and rename it \(Y\), denoting the remaining \(d\) variables with \(\*Z = (Z_1, \dots, Z_d)\). Assume an SCM defined over the random vector \((\*Z, Y)\) inducing the joint distribution \(P\) over \((\*Z, Y)\). Assume \(Y\) is not a parent of any \(\*Z\). Since \(P\) is induced by an SCM, it factorizes as
\begin{align}
    P(\*Z, Y) = \prod^{d+1}_{i=1} P(Z_i \mid Z_{\Pa_i}),
\end{align}
where \(\Pa_i \subset [d+1] \setminus i\) denotes the parents of variable \(Z_i\). We refer to \citep{pearl_causality_2009} for an in-depth definition of SCMs. 

Beyond the additive noise assumption we do not place any parametric constraints on the causal mechanism of each variable, i.e. each structural equation can be written as \(Z_i := f_i(Z_{\text{Pa}_i}) + \varepsilon_i\), where \(\varepsilon_i; i \in [d+1]\) denote the exogenous, independent noise terms, which are assumed to have zero mean. 
Since the causal mechanism of the target \(Y\) is of particular interest in a predictive setting, we denote it with \(f_{\text{causal}} := f_{d+1}\).

We do not directly measure \(\*Z\) and only assume to observe \(\*X \in \mathbb{R}^p\), where \(\*X = \g(\*Z)\)
is a transformation of the causal variables \(\*Z\) by the injective, deterministic (and potentially nonlinear) function \(\g\). We assume \(p \geq d\). Notice that \(Y\) is \emph{not} transformed by \(\g\); we assume that the target variable is directly observed. The set of observed variables is therefore denoted by \((\*X, Y)\). \cref{fig:dgp} depicts a graphical representation of this data generating process.

We assume to observe \((\*X, Y)\) across multiple environments, where subsets \(S \subseteq [d]\) of the underlying, latent variables \(\*Z\) have undergone an intervention in each environment. We model interventions with do-interventions \citep{pearl_causality_2009} (also called hard interventions), which set the structural equations of the variables that are targeted by an intervention to constant values, allowing us to write
\begin{align}
        Z_j := a_j \quad \text{for } j \in S,
\end{align}
where \(\*a \in \mathbb{R}^{\left|S\right|}\). Each intervention on the subset of variables \(S\) with value \(\*a\) induces a new distribution over \((\*Z, Y)\) which, following the notation introduced by \citet{meinshausen_causality_2018}, is denoted by \(P^{(\text{do})}_{\*a, S}\).
We assume \(Y\) is never among the set of intervention targets and the mixing function \(\g\) does not change between environments.

\paragraph{Objective.}
Although we explicitly define a target variable, we stress that our main goal is not to learn a good predictive model of \(Y\), but rather to recover the latent variables \(\*Z\), i.e. we focus on the representation learning aspect of the setting described above. Our aim is to probe how the auxiliary task of predicting \(Y\) from transformations of latent variables in multiple environments can help in recovering the unobserved causal variables.
Our formal objective is to invert the mixing function \(\g\) in order to recover the latent variables \(\*Z\) from observations \(\*X\).
Since latent variables that are equal to the ground truth up to permutation and element-wise rescaling can give rise to the same observations \(\*X\) \citep{zhang_identifiability_2023}, we define an equivalence relation over this class of latents. Consequently, we define recovering the latent causal variables up to this equivalence as our notion of identifiability.
Equivalent representations are referred to as \emph{causally disentangled}.
\begin{definition}[Causally Disentangled Representations, \citet{khemakhem_variational_2020,lachapelle_disentanglement_2022}]\label{def:disent}
    A learned representation \(\hat{\*Z}\) is causally disentangled w.r.t. to the ground truth representation \(\*Z\) if there exists an invertible diagonal matrix \(\*D\) and a permutation matrix \(\*P\), s.t. \(\hat{\*Z} = \*D \*P \*Z\) almost surely.
\end{definition}

\section{Invariance for causal structure learning}
The connection between (predictive) invariance and causality is long established: \citet{haavelmo_probability_1944} was the first to formalize that a model which predicts a target from its direct causes is invariant under interventions on any other covariates of the system. In the language of SCMs
this means that the conditional distribution
\(P(Y|\*Z_{\Pa_Y})\)
remains invariant under any interventions on \(\*Z\). This principle is also referred to as autonomy \citep{aldrich_autonomy_1989}, modularity \citep{pearl_causality_2009} and independence of cause and mechanism \citep{peters_elements_2017}.

More recently, the opposite direction has been explored, namely how invariance can be leveraged as a signal to infer causal structures and mechanisms, pioneered by the work of \citet{peters_causal_2016}, who exploit the principle of invariance of causal mechanisms to infer the direct causes of a target \(Y\), assuming direct observations of the causal variables \(\*Z\). A particularly interesting line of works draws a connection between distributional robustness and causality \citep{meinshausen_causality_2018, buhlmann_invariance_2018} by considering the problem
\begin{align}\label{eq:rob_opt}
    \min_f \sup_{Q \in \mathcal{Q}} \mathbb{E}_{(\*Z, Y) \sim Q}\left[(Y - f(\*Z))^2\right],
\end{align}
where \(\mathcal{Q}\) denotes some set of interventional distributions. \citet{rojas-carulla_invariant_2018} establish that the causal predictor \(\f\) is a solution for adversarially chosen \(\mathcal{Q}\), \citet{christiansen_causal_2022} investigate under which choices of \(\mathcal{Q}\), \(\f\) remains a solution and \citet{meinshausen_causality_2018} states that \(\f\) is the \emph{unique} solution to this problem when \(\f\) is linear and \(\mathcal{Q}\) is the set of interventions on \emph{all} variables except \(Y\), with arbitrary strength. None of the mentioned works consider settings where \(\*Z\) is latent.

As a first step towards our main theoretical findings, we extend the results of \citet{meinshausen_causality_2018} by showing that the causal mechanism of the target \(\f\) is the unique solution to \cref{eq:rob_opt} for general, nonlinear \(\f\) when the set of interventions \(\mathcal{Q}\) contains interventions on all covariates \(\*Z\). Assuming interventions on all variables in \(\*Z\) can be understood as a diversity condition on the observed environments.

\begin{lemma}[]\label{lem:rob_opt}
    Assume the general SCM presented in \cref{sec:prob} as the data generating process and consider the robust optimization problem described by \cref{eq:rob_opt}. Let \(\mathcal{Q}^{(\textnormal{do})} := \bigr\{P^{(\textnormal{do})}_{\*a,[d]}; \*a \in \mathbb{R}^d\bigr\},\) i.e. the set of do-interventions on \emph{all} variables, except \(Y\), with arbitrary strength.
    Then, the causal mechanism of the target \(Y\), \(\f\), is the unique optimizer of \cref{eq:rob_opt}, i.e.
    \begin{align*}
        f_{\textnormal{causal}} = \argmin_f \sup_{Q \in \mathcal{Q}^{(\textnormal{do})}} \mathbb{E}_{(\*Z, Y) \sim Q}\left[(Y - f(\*Z))^2\right].
    \end{align*}
\end{lemma}
\begin{proof}
    Suppose that \(f \neq \f\), which implies that there exists at least one \(z = (z_1, \dots, z_d)\) such that \(f(z) \neq \f(z)\).
    Now, consider the decomposition of the objective
    \begin{align*}
        \mathbb{E}_Q\bigr[(Y-f(\*Z))^2\bigr] &= \mathbb{E}_Q\bigr[(\f(\*Z)-f(\*Z)+\varepsilon_Y)^2\bigr]\\\
        &= \mathbb{E}_Q\bigr[(\f(\*Z)-f(\*Z))^2\bigr] + \mathbb{E}_Q\bigr[\varepsilon_Y^2\bigr] + 2 \mathbb{E}_Q\bigr[(\f(\*Z)-f(\*Z))\varepsilon_Y\bigr].
    \end{align*}
    For any interventional distribution \(Q \in \mathcal{Q}^{(\text{do})}\) where all \(Z_j\)s are intervened upon, we notice that \(\varepsilon_Y\) is independent of all \(Z_j\)'s. Hence, for any such interventional distribution 
    \begin{align*}
        2 \mathbb{E}_Q\bigr[(\f(\*Z)-f(\*Z))\varepsilon_Y\bigr] = 2(\f(\*a)-f(\*a))\mathbb{E}_Q\bigr[\varepsilon_Y\bigr] = 0,
    \end{align*}
    i.e. the last term in the decomposition vanishes. Here we used that \(\varepsilon_Y\) has expectation zero. Next, we focus on the first term in the decomposition. We want to find an interventional distribution \(Q\) s.t.
    \begin{align*}
        \mathbb{E}_Q\bigr[(\f(\*Z)-f(\*Z))^2\bigr] >0.
    \end{align*}
    To do so, we simply choose \(\*a\) such that \(\f(\*a) \neq f(\*a)\). We know such a choice of \(\*a\) exists since \(\f\) and \(f\) are not equal by assumption. For this particular intervention we have
    \begin{align*}
        \mathbb{E}_Q\bigr[(\f(\*Z)-f(\*Z))^2\bigr] = \mathbb{E}_Q\bigr[(\f(\*a)-f(\*a))^2\bigr] = (\f(\*a)-f(\*a))^2 > 0.
    \end{align*}
    Thus, 
    \begin{align*}
        \mathbb{E}_{Q}[(Y-f(\*Z))^2]>Var(\varepsilon_Y),
    \end{align*}
    for any $Q \in \mathcal{Q}^{(\text{do})}$. The supremum is therefore also strictly larger than \(Var(\varepsilon_Y)\).
    For \(\f = f\) 
    \begin{align*}
        \mathbb{E}_{Q}[(Y-f(\*Z))^2]=Var(\varepsilon_Y),
    \end{align*}
    which also holds for the supremum, and we conclude that \(\f\) is the unique optimizer of \cref{eq:rob_opt}.
\end{proof}




As we are interested in establishing identifiability results for representation learning, this uniqueness result provides a potentially fruitful starting point. Notice that the above result assumes direct access to the variables \(\*Z\), while our problem setting of interest is characterized by the central assumption of only observing transformations \(\*X = \g(\*Z)\) of the underlying causal variables. We investigate the implications of this transformation in the next section.

\section{Invariance for causal representation learning}\label{sec:inv_crl}
To finally arrive at the problem setting we are interested in, we introduce the representation function \(g\) to the optimization problem presented in \cref{eq:rob_opt}, recalling that \(\*X=\g(\*Z)\). Now, consider the extended problem
\begin{align}\label{eq:rob_opt_mix}
    \min_{f,g} \sup_{Q \in \mathcal{Q}} \mathbb{E}_{(\*Z, Y) \sim Q}\left[(Y - f(g^{-1}(\*X)))^2\right].
\end{align}
Notice that the expectation is still taken over the joint distribution of \((\*Z, Y)\), i.e. the different environments which we consider still arise from interventions on the underlying latent variables \(\*Z\).

If solving \cref{eq:rob_opt} allows us to uniquely recover \(\f\), can solving the extended problem in \cref{eq:rob_opt_mix} allow us to draw similar conclusions about \(\g\)?

\subsection{Causal mechanism and representation are jointly unique}
As a first uniqueness result, we show that the joint function composed of the causal mechanism \(\f\) and the inverse of the ground-truth representation function \(\g^{-1}\) is the unique solution to \cref{eq:rob_opt_mix}.
\begin{lemma}[]\label{lem:rob_opt_mix}
    Assume the data generating process in \cref{sec:prob} and consider the  optimization problem described in \cref{eq:rob_opt_mix}. Let \(\mathcal{Q}^{(\textnormal{do})} := \bigr\{P^{(\textnormal{do})}_{\*a,[d]}; \*a \in \mathbb{R}^d\bigr\},\) i.e. the set of do-interventions on \emph{all} underlying variables \(\*Z\), except \(Y\), with arbitrary strength. Define \(h := f \circ g^{-1}: \mathbb{R}^p \to \mathbb{R}\) and let \(\Ima(\cdot)\) denote the image of a function.
    Then, the composed function
    \(\h := (\f \circ \g^{-1})\) is the unique optimizer of \cref{eq:rob_opt_mix} on \(\Ima(\g)\), i.e.
    \begin{align*}
        \h = (f_{\textnormal{causal}} \circ g^{-1}_{\textnormal{causal}}) \in \argmin_h \sup_{Q \in \mathcal{Q}^{(\textnormal{do})}} \mathbb{E}_{(\*Z, Y) \sim Q}\left[(Y - h(\*X))^2\right],
    \end{align*}
    and any other minimizer \(h'\) that satisfies \(h' \equiv \h\) on \(\Ima(\g)\).
\end{lemma}
To prove this Lemma, we need to find a \(\*b\) s.t. \(h(\*b) \neq \h(\*b)\). Notice that \(\h\) takes as argument \(\*X = \g(\*Z)\) while we can only intervene directly on \(\*Z\). Thus, we can only find \(\*b \in \Ima(\g)\) and consequently the statement \(h(\*b) \neq \h(\*b)\) only holds on the image of \(\g\). Taking this into account, the proof (presented in \cref{app:proof}) follows the same structure as the proof of \cref{lem:rob_opt}.

\subsection{Overparametrization of the unconstrained setting}\label{ssec:overparam}
While the previous uniqueness result regarding the composed function \((\f \circ \g^{-1})\) is a first promising direction in relating the solution of a distributional robustness problem to inverting the representation function \(\g\), we quickly see that this uniqueness result does not yet constrain the individual components \(f\) and \(g\) of the solution enough.

\begin{theorem}[]\label{thm:imp1}
    Consider the data generation process presented in \cref{sec:prob}. Without additional assumptions, the distributional robustness problem described by \cref{eq:rob_opt_mix} does not suffice to identify the underlying causal variables up to the equivalence class detailed in \cref{def:disent}.
\end{theorem}
\begin{proof}  

    While \((\f \circ \g^{-1})\) has been shown to be the unique solution to the optimization problem described by \cref{eq:rob_opt_mix}, this does not directly imply the respective uniqueness of its components \(f\) and \(g\). To see this, consider any invertible map \(\Psi : \mathbb{R}^d \rightarrow \mathbb{R}^d\) and write
    \begin{align*}
        \f \circ \g^{-1} = \underbrace{\f \circ \Psi^{-1}}_{:=\hat{f}} \circ \underbrace{\Psi \circ \g^{-1}}_{:= \hat{g}^{-1}}.
    \end{align*}
    Thus, the tuple  \((\hat{f},\hat{g})\) with \(\hat{f} := \f \circ \Psi^{-1}\) and \(\hat{g}^{-1} := \Psi \circ \g^{-1}\) also gives rise to the solution $h_{\text{causal}} = \f \circ \g^{-1} = \hat{f} \circ \hat{g}^{-1}$ of \cref{eq:rob_opt_mix}. 
    
    Our goal is to recover \(\g^{-1}\) up to the equivalence class described in \cref{def:disent}, but we see that \(\hat{g}^{-1} = \Psi \circ \g^{-1}\) is a solution to our considered problem, where \(\Psi\) can be \emph{any} invertible transformation. We therefore conclude that our considered problem setting is underconstrained and we cannot identify \(\g^{-1}\) without additional assumptions.
\end{proof}
This result is unsurprising, given that we add a degree of freedom to the original problem in \cref{eq:rob_opt}, in the form of the function \(g\), without adding further constraints. 
As a result our problem becomes overparametrized and we can no longer uniquely recover both functions \(\f\) and \(\g\).

Note that if our goal is to find a predictive model that maps observations \(\*X\) to a target \(Y\), which is robust to distribution shifts, this impossibility result is not an issue, since only the \emph{composition} \((f \circ g^{-1})\) matters to solve \cref{eq:rob_opt_mix}.
Similar findings are shown in \citep{arjovsky_invariant_2020}, underlining that predicting optimally under distribution shift does not require the causal representation.

\subsection{Necessity of additional assumptions}\label{ssec:lin_f}
So far, we have not imposed any parametric constraints on \(\f\) or \(\g\). The impossibility result described in \cref{thm:imp1} however implies that we require further assumptions to make progress towards identifying the latent causal variables. Since we want our results to hold for general \(\g\), we refrain from beginning with assumptions on the representation function. Rather, we will investigate how parametric assumptions on \(\f\) may be used to constrain class of functions \(g\) that solve \cref{eq:rob_opt_mix}. 

Notice that the functions that solve \cref{eq:rob_opt_mix}, \(\hat{f} = \f \circ \Psi^{-1}\) and \(\hat{g}^{-1} = \Psi \circ \g^{-1}\), cannot be chosen independently of each other, but are connected via the map \(\Psi\). This connection motivates our reasoning behind constraining \(\f\): for certain parametric choices of \(\f\) (and accordingly \(\hat{f}\)) perhaps only a constrained set of maps \(\Psi\) admits a solution to \cref{eq:rob_opt_mix}. If this is the case and we effectively constrain \(\Psi\), we might be able to find a constraint on \(\f\) such that only transformations of the form \(\Psi = \*D \*P\), where \(\*D\) is a diagonal matrix and \(\*P\) is a permutation matrix, which would result in recovering the causally disentangled \(\hat{g}^{-1} = \*D \*P \g^{-1}\), according to \cref{def:disent}.

\paragraph{Linear causal mechanism.}
A first natural assumption to impose is linearity of \(\f\), and correspondingly of \(\hat{f}\), in the hopes of constraining \(\Psi\) to be a linear invertible map. This would allow us to make substantial progress towards recovering the causal variables up to permutation and rescaling, by first recovering the ground truth representation up to linear equivalence and then employing tactics to undo this linear mixing, following a common approach in causal representation learning \citep{ahuja_interventional_2023,zhang_identifiability_2023,lachapelle_synergies_2023}.

We explore the implications of assuming linearity of \(\f\) with an illustrative example. Assume \(Y := \theta_1 Z_1 + \theta_2 Z_2 + \varepsilon_Y\), i.e. the two variable case where \(d=2\). Since we assume \(\f\) to be linear, we also constrain the search space of \(\hat{f}\) to the class of linear functions. Recall that \(\hat{f} := \f \circ \Psi^{-1}\). Does the assumed linearity of \(\hat{f}\) and \(\f\) now constrain \(\Psi\) to also be a linear map? Writing out
\begin{align*}
    \hat{f} = \f \circ \Psi^{-1} = \theta_1 \psi^{-1}_1 + \theta_2 \psi^{-1}_2,
\end{align*}
where \(\psi^{-1}_i\) denotes the \(i\)th component of \(\Psi^{-1}\), we see that e.g. by choosing \(\theta_1=1, \theta_2=0\) only \(\psi^{-1}_1\) is constrained to be linear, while \(\psi^{-1}_2\) remains wholly unconstrained. We generalize this result to arbitrary choices of \(\theta\) in the following theorem.

\begin{theorem}[]\label{thm:imp2}
    Assume that \(\*Z,\*X,Y,\f,\g\) follow the definitions in \cref{sec:prob} with $d \geq 2$ and that additionally $0 \neq \f: \mathbb R^d \to \mathbb R$ is a linear function. Then, there exists an invertible nonlinear function $\Psi: \mathbb R^d \to \mathbb R^d$ such that $\hat{f} = \f \circ \Psi$ is linear. Moreover, this function $\Psi$ is not unique and can be chosen arbitrarily (that is, only constrained to be invertible) on a $d-1$-dimensional subspace of $\mathbb R^d$. In particular, the tuple $(\hat{f},\hat{g})$ with $\hat{g}:= \g \circ \Psi$ gives rise to the solution $h_{\text{causal}}$ of \cref{eq:rob_opt_mix} in that $h_{\text{causal}} = \hat{f} \circ \hat{g}^{-1}$.  
\end{theorem}
\begin{proof}
    We begin with an example which will later be generalized. Our goal is to elucidate if assuming \(\hat{f}\) and \(\f\) to be linear functions necessarily constrains \(\Psi\) to be a linear map, where \(\hat{f} := \f \circ \Psi^{-1}\). We write $\f(\mathbf{z}) = \theta^T \mathbf{z} $, where \(\theta \in \mathbb{R}^d\) and consider the case where \(\theta = (1, 0, \dots, 0)^T\). Consider any arbitrary invertible map \(\Psi: \mathbb{R}^d \rightarrow \mathbb{R}^d\) with components \(\psi_i\). We now write \(\hat{f}\) as
    \begin{align*}
        \hat{f} = \f \circ \Psi^{-1} = \theta_1 \psi^{-1}_1 + \dots + \theta_d \psi^{-1}_d = \psi^{-1}_1.
    \end{align*}
    For this particular choice of \(\theta\), constraining $\hat{f}$ to be linear amounts to constraining \(\psi^{-1}_1\) to be linear, while the other components \(\psi^{-1}_i; i \in [d]\setminus\{1\}\) are not constrained at all and can be chosen arbitratrily as long as \(\Psi\) remains invertible. For example, the map $\Psi(\mathbf{z}) = (z_1,z_2^3,\dots,z_d^3)$ does the job.

    For a general choice of \(\theta\) we can always find an orthonormal transformation \(\*A: \mathbb{R}^d \rightarrow \mathbb{R}^d\) that maps \((\theta_1, \dots, \theta_d)^T \mapsto (1, 0, \dots, 0)^T\). Consider a nonlinear map $\Psi_0$ for the initial case  \(\theta = (1, 0, \dots, 0)^T\), whose first component is linear. Define \(\Psi' := \*A^T \circ \Psi\) which remains nonlinear. We can now write
    \begin{align*}
        \hat{f}(\mathbf{z}) := \left(\f \circ \Psi \right)(\mathbf{z}) &= \theta^T \left( \*A^T \Psi_0(\mathbf{z})  \right) = (\*A \theta) \Psi_0(\mathbf{z}) =  (1, 0, \dots, 0)^T \Psi_0(\mathbf{z})
    \end{align*}
    Thus $\hat{f}$ is linear while $\Psi$ was only constrained to be linear in its first component.
\end{proof}
Through the generalized counterexample presented in the proof of \cref{thm:imp2}, we see that the linearity requirement on \(\f \circ \Psi^{-1}\) only constrains one dimension of \(\Psi\), namely the one orthogonal to the basis of \(\f\). We would need one such constraint for each dimension of \(\Psi\)---possibly by assuming \(\dim(Y) \geq d\) as in \citep{lachapelle_synergies_2023}---in order to draw the desired linearity conclusion.

Unfortunately, the---arguably strong---assumptions presented in this section still do not suffice to recover \(\g\) up to the desired equivalence class. In the following sections, we delineate possible steps forward, both in light of the presented impossibility results, as well as practical considerations.

\subsection{Practical considerations}\label{ssec:prac}
So far, we have considered an idealized setting with infinite data, perfect optimization and most importantly an infinite number of environments, stemming from interventions with arbitrary strength \(\*a \in \mathbb{R}^d\). While not uncommon assumptions in general, in practice we obviously do not have access to infinitely many environments. Rather, we consider a finite set of training environments, over which we formulate our invariance condition, in hopes of generalizing to unseen environments.

Training a model on finite support and generalizing outside of this support amounts to extrapolation. As \citet{christiansen_causal_2022} show, learning such \emph{extrapolating} nonlinear functions from data with bounded support necessarily requires strong assumptions on the function class. If we do not constrain \(\g\), even if \(\f\) is linear, the function \(\f \circ \g^{-1}\), that should generalize in the aforementioned sense, is still generally nonlinear. Linear functions however do have this desired extrapolation property, rendering them interesting candidates for generalization, also from a practical perspective.

\section{Outlook}
Given the proven insufficiency of the assumptions in \cref{ssec:overparam,ssec:lin_f} along with the practical limitations of learning nonlinear functions that generalize outside of the training data detailed in \cref{ssec:prac}, the framework of utilizing invariance as a learning signal for causal representations seems ill-equipped without additional inductive biases.

To achieve identifiability, one could impose the linearity constraint proposed in \cref{ssec:lin_f} on each dimension of \(\Psi\), however this constraint essentially amounts to directly assuming the linearity of \(\Psi\), which is the implication we are trying to show in the first place. Alternatively, by assuming \(\dim(Y) \geq d\) one could directly use the results of \citet{lachapelle_synergies_2023} to achieve identifiability, albeit without exploiting any kind of invariance or interventional data, which is the core motivation of this study.

In order to overcome the challenges staked out in this work, we propose to move forward by considering the case where \(\g\) is a linear map. Establishing theoretical results in the context of distributional robustness problems in conjunction with representation learning and out-of-distribution prediction seems to hinge on some type of linearity assumption \citep{peters_causal_2016, arjovsky_invariant_2020, rosenfeld_risks_2020, krueger_out--distribution_2021, eastwood_probable_2022, lachapelle_synergies_2023}, motivating the assumption of linear predictive mechanism \(\f\) in conjunction with a linear mixing function \(\g\). While directly assuming linearity of \(\g\) is a strong assumption, we argue that such a setting still bears practical merit. A number of works yield linearly mixed representations \citep{roeder_linear_2021, ahuja_interventional_2023, lachapelle_synergies_2023, saengkyongam_identifying_2023}, any of which can serve as the starting point for a method that considers linear mixing functions.

Additionally assuming linearity of \(\g\) does not directly lead to identifiability of the latent variables, as this does not alleviate the invariance of the solution of \cref{eq:rob_opt_mix} to reparametrizations by an invertible map \(\Psi\).  
However, we postulate that this linear setting allows us to utilize similar assumptions as other works that deal with linear mixing of the causal variables \citep{zhang_identifiability_2023, squires_linear_2023, buchholz_learning_2023, varici_score-based_2023, lachapelle_synergies_2023}, in order to establish both identifiability results and practical algorithms. A particular set of candidate assumptions that are being explored as part of ongoing work are sparsity constraints similar to those in \citep{lachapelle_disentanglement_2022, lachapelle_synergies_2023}.

\paragraph{Conclusion.} We have presented a theoretical investigation into the potential of predictive invariance, characteristic to causal models, as a signal for representation learning. We proposed a formal framework that allows us to approach this question and provide first impossibility results demarcating the strength of necessary assumptions towards identifiability. Finally, we shed light on further practical challenges that arise and propose constraints derived from these challenges, which we hypothesize will help to make progress towards utilizing invariance for latent causal variable recovery.

\section*{Acknowledgements}
We thank Tom Hochsprung for insightful discussions and comments, as well as the anonymous reviewers for their feedback.
This work has received funding from the European Research Council (ERC) Starting Grant CausalEarth under the European Union’s Horizon 2020 research and innovation program (Grant Agreement No. 948112). 

\bibliography{references}


\newpage
\appendix

\section{Proof of \cref{lem:rob_opt_mix}}\label{app:proof}
\setcounter{lemma}{1}
\begin{lemma}[]
     Assume the data generating process in \cref{sec:prob} and consider the  optimization problem described in \cref{eq:rob_opt_mix}. Let \(\mathcal{Q}^{(\textnormal{do})} := \bigr\{P^{(\textnormal{do})}_{\*a,[d]}; \*a \in \mathbb{R}^d\bigr\},\) i.e. the set of do-interventions on \emph{all} underlying variables \(\*Z\), except \(Y\), with arbitrary strength. Define \(h := f \circ g^{-1}: \mathbb{R}^p \to \mathbb{R}\) and let \(\Ima(\cdot)\) denote the image of a function.
    Then, the composed function
    \(\h := (\f \circ \g^{-1})\) is the unique optimizer of \cref{eq:rob_opt_mix} on \(\Ima(\g)\), i.e.
    \begin{align*}
        \h = (f_{\textnormal{causal}} \circ g^{-1}_{\textnormal{causal}}) \in \argmin_h \sup_{Q \in \mathcal{Q}^{(\textnormal{do})}} \mathbb{E}_{(\*Z, Y) \sim Q}\left[(Y - h(\*X))^2\right],
    \end{align*}
    and any other minimizer \(h'\) that satisfies \(h' \equiv \h\) on \(\Ima(\g)\).
\end{lemma}
\begin{proof}
    This proof largely follows the proof of \cref{lem:rob_opt}.

    Consider the decomposition of the objective
    \begin{align*}
        \mathbb{E}_Q\bigr[(Y - h(\*X))^2\bigr] =& \mathbb{E}_Q\bigr[(Y - h(\g(\*Z)))^2\bigr] \\
        = &\mathbb{E}_Q\bigr[(\f(\*Z) - h(\g(\*Z)) + \varepsilon_Y)^2\bigr] \\
        = &\mathbb{E}_Q\bigr[(\f(\*Z) - h(\g(\*Z)))^2\bigr] + \mathbb{E}_Q\bigr[\varepsilon^2_Y\bigr] \\
        &+ 2\mathbb{E}_Q\bigr[(\f(\*Z) - h(\g(\*Z)))\varepsilon_Y\bigr].
    \end{align*}
    Again, for any interventional distribution \(Q \in \mathcal{Q}^{(\text{do})}\) where we intervene on all \(Z_j\)s, we see that \(\varepsilon_Y\) is independent of all variables \(Z_j\) and the final term of the above decomposition vanishes
    \begin{align*}
        2\mathbb{E}_Q\bigr[(\f(\*Z) - h(\g(\*Z)))\varepsilon_Y\bigr] &= 2 (\f(\*a) - h(\g(\*a))) \mathbb{E}_Q\bigr[\varepsilon_Y\bigr] \\
        &=  2 (\f(\*a) - h(\g(\*a))) \mathbb{E}_{P_{\varepsilon_Y}}\bigr[\varepsilon_Y\bigr] \\
        &= 0,
    \end{align*}
    where we use the fact that \(\varepsilon_Y\) has mean zero. Hence, for any choice of \(h\) the supremum is always larger or equal to \(Var(\varepsilon_Y)\).

    Focusing our attention on the first term of the decomposition presented above, suppose \(h \neq \h\), i.e. there exists a choice of \(\*b\) s.t. \(h(\*b) \neq \h(\*b)\). Recall that \(h\) takes \(\*X\) as its argument, but we cannot directly intervene on \(\*X = \g(\*Z)\), only on \(\*Z\). 
    
    Therefore, assume \(\*b \in \Ima(\g)\) and consider an interventional distribution  \(Q\) where we choose \(\*a\) s.t. \(\*b = \g(\*a)\). Then
    \begin{align*}
        \mathbb{E}_Q\bigr[(\f(\*Z) - h(\g(\*Z)))^2\bigr] &= \mathbb{E}_Q\bigr[(\f(\*a) - h(\*b))^2\bigr] \\
        &= \mathbb{E}_Q\bigr[(\h(\g(\*a)) - h(\*b))^2\bigr] \\
        &= \mathbb{E}_Q\bigr[(\h(\*b) - h(\*b))^2\bigr] \\
        &>0,
    \end{align*}
    where we used \(\f = \h \circ \g\) in the second line.

    Conversely, for any function \( h\) that coincides with \(\h\) on \(\Ima(\g)\) the inequality above becomes an equality with zero, rendering any such function an optimizer for the considered problem.
\end{proof}

\section{Detailed related work}
\paragraph{Invariance, distributional robustness and causality.} As detailed in the main text, the principle of invariance is closely linked to ideas from causality. The first work that proposed to use this invariance principle to learn causal structures from data---and kicked off a range of subsequent works---was Invariant Causal Prediction (ICP) by \citet{peters_causal_2016}, where the fact that a predictive model conditioned on all parents of a target \(Y\) is invariant under interventions is used to find the parents of said target. \citet{pfister_learning_2019} show the merit of using invariance as a signal for selecting robust prediction models for real-world biological data.
\citet{magliacane_domain_2018} utilize invariance and the Joint Causal Inference framework \citep{mooij_joint_2020} to find features that lead to transferable predictions across contexts, without relying on knowledge of the causal graph or specific types of interventions.
\citet{buhlmann_invariance_2018} and \citet{meinshausen_causality_2018} connect causality and distributional robustness and propose an alternative view where causal structures and models are defined as those that induce invariance. In Anchor Regression, \citet{rothenhausler_anchor_2021} propose a regression model that solves a distributional robustness problem by enforcing invariance to a specific type of shift interventions. \citet{christiansen_causal_2022} characterize various distributional robustness problems, investigating the influence of the functional class of the prediction model and whether interventions extend the support of training data.

\paragraph{OOD generalization.} Distributional robustness problems such as \cref{eq:rob_opt,eq:rob_opt_mix} are closely related to out-of-distribution (OOD) generalization in machine learning and \citet{rojas-carulla_invariant_2018} show that predictors that use the direct causes of a target are optimal for certain OOD problem settings. Another line of works, Risk Extrapolation (REx) \citep{krueger_out--distribution_2021} and Quantile Risk Minimization (QRM) \citep{eastwood_probable_2022}, assume a slightly different type of invariance, namely that of the risk \(\mathcal{R} := \mathbb{E}[\ell(Y, \*X)]\), where \(\ell\) is some loss function. REx proposes to extrapolate the convex hull of risks encountered in training to achieve robustness in test time. QRM also posits invariance of risks, but assumes a probabilistic point of view and aims not to find worst-case predictors, but those that perform well with high probability. Under additonal technical assumptions, both approaches prove that they can recover the causes of a target \(Y\) in a linear SCM, if all causal variables are directly observed.

Invariant Risk Minimization (IRM) \citep{arjovsky_invariant_2020} is another approach to OOD generalization, specifically geared towards formalizing this problem in the context of machine learning. As such, the authors consider a similar setting to ours, where we assume some underlying latent variables that permit the formulation of an invariant predictor, together with a function that maps these latents to the observations we have access to. IRM however is not interested in representation learning, as we are, but is solely aimed at learning invariant prediction models. Beyond showing, for the linear case, that IRM can separate the part of the representation that permits an invariant predictor from those parts that do not, the authors do not provide theoretical results pertaining to the identification of the underlying variables. For the nonlinear setting, follow up works have demonstrated that IRM is ill-equipped at learning predictors that perform OOD generalization \citep{kamath_does_2021,rosenfeld_risks_2020}, echoing the considerations we bring forth in \cref{ssec:prac}.

\paragraph{Causal representation learning.} 
Initiated by the famous impossibility result of nonlinear independent component analysis (ICA) \citep{hyvarinen_nonlinear_1999}, unsupervised representation learning has been shown to be too underconstrained to be solved without additional inductive biases \citep{locatello_challenging_2019}. Later works on the identifiability of ICA problems have been able to overcome this initial obstacle by exploiting various types of auxiliary assumptions \citep{hyvarinen_nonlinear_2017,halva_hidden_2020, gresele_independent_2021,morioka_connectivity-contrastive_2023}, and more recent works in causal representation have followed this example, too.
Often, heterogeneity of the observed data distribution is made, e.g. arising from counterfactual pairs \citep{brehmer_weakly_2022} or interventions, specifically hard do-interventions as in \citep{ahuja_interventional_2023} or soft interventions as in \citep{zhang_identifiability_2023}. 
Other approaches shift their focus to the mixing function that transforms the underlying causal variables, with one family of works focusing on how to deal with linear mixtures.
\citet{squires_linear_2023} consider linear mixtures of linear SCMs, \citet{buchholz_learning_2023} generalize this result to nonparametric SCMs, and \citet{varici_score-based_2023} focus on learning the representation of linear mixtures of causal variables via a score based approach. Alternatively, some works look to exploit temporal structures, such as \citet{lippe_citris_2022,lippe_causal_2022,lippe_biscuit_2023} who use knowledge of interventions to achieve identifiability, \citet{yao_learning_2021,yao_temporally_2022} who exploit nonstationarity or \citet{lachapelle_disentanglement_2022} who use an assumption about the sparsity of the SCM that generates the data.

Another approach that exploits a specific sparsity assumption is given by \citet{lachapelle_synergies_2023}. Similar to our setting, the synergy of representation learning with a prediction problem is explored. Instead of considering interventions that induce heterogeneity in the data, this framework assumes multiple prediction tasks where a single, underlying representation contains the potential predictors for each individual task. 
This setting alone only suffices to recover the latent variables up to linear mixing, but by imposing an additional sparsity constraint, the latents are shown to be causally disentangled (cf. \cref{def:disent}).

Works considering latent DAG structure learning are also in spirit related to our setting, as they commonly assume to observe at least some nodes within the graph they aim to learn, similar to how we assume observations of the target \(Y\). 
The main assumption for most works in this setting is the pure children assumption \citep{silva_learning_2006}, or some variation thereof. This assumption postulates that all observed variables have only a single latent variable as a parent. \citet{cai_triad_2019} deal with linearly transformed latent variables that have two pure children, subsequently generalized by \citet{xie_generalized_2020} to permit more than two pure children per latent in the form of the so-called Generalized Independent Noise condition. The same authors' later work further generalizes this setting to facilitate the learning of hierarchical latent variable models \citep{xie_identification_2022}.
While our approach also assumes the observation of an effect of the underlying SCM in the form of \(Y\), the main difference to the aforementioned works lies in the modelling choice of the remaining observed variables, \(\*X\). In our case, we consider \(\*X\) to come from an injective transformation of the underlying variables \(\*Z\), and therefore not to be causal variables of any SCM directly, while the methods mentioned above consider all observed variables \(\*X\) to be part of the underlying, latent SCM and consequently model their relation to the latent variables \(\*Z\) in terms of surjective causal mechanisms. Given this difference, we do not require assumptions on the mechanism of \(Y\) similar to those in \citep{xie_generalized_2020} and related works pertaining to the number of observed children of latents \(\*Z\) or the linearity of the transform that yields \(\*X\).

\end{document}